\documentclass[lettersize,journal]{IEEEtran}
\usepackage{amsmath,amsfonts}
\usepackage{amsthm}
\usepackage{algorithmic}
\usepackage{array}
\usepackage{textcomp}
\usepackage{stfloats}
\usepackage{url}
\usepackage{verbatim}
\usepackage{subfloat}
\usepackage{booktabs} 
\usepackage{subfigure}
\usepackage{cite}
\usepackage{graphicx}
\usepackage{algorithm,algorithmic}
\def\BibTeX{{\rm B\kern-.05em{\sc i\kern-.025em b}\kern-.08em
    T\kern-.1667em\lower.7ex\hbox{E}\kern-.125emX}}
\usepackage{balance}

\usepackage{tabularx}

\usepackage{hyperref}
\hypersetup{hypertex=true,
colorlinks=true,
linkcolor=blue,
anchorcolor=blue,
citecolor=blue}
\begin{document}
\title{Fed-DLoRA: Efficient Wireless Federated Learning with Dynamic Low-Rank Adaptation}
\author{
        Huaicheng Li,
        Junhui Zhao,~\IEEEmembership{Senior Member, IEEE},
        Haoyu Quan,
        Xiaoming Wang
    \thanks{Copyright (c) 2026 IEEE. Personal use of this material is permitted. However, permission to use this material for any other purposes must be obtained from the IEEE by sending a request to pubs-permissions@ieee.org.}
    
	\thanks{This work was supported by the Fundamental Research Funds for the Central Universities 2025JBZX060
        \textit{(Corresponding author: Junhui Zhao.)}}
	\thanks{Huaicheng Li, Junhui Zhao and Haoyu Quan are with the School of Electronic and Information Engineering, Beijing Jiaotong University, Beijing 100044, China (e-mail: huaichengli@bjtu.edu.cn, junhuizhao@hotmail.com, koterial@hotmail.com).
    
    Xiaoming Wang is with the School of Communication and Information Engineering, Nanjing University of Posts and Telecommunications, Nanjing 210003, China(email: xmwang@njupt.edu.cn)}
}


\maketitle

\begin{abstract}
Federated learning (FL) offers a promising distributed learning paradigm for internet of vehicles (IoV) applications.
However, it faces challenges from communication overhead and dynamic environments. 
Model compression techniques reduce computing and communication burden yet create trade-offs between compression ratios and vehicle participation strategies.
In this paper, we propose a lightweight FL algorithm named federated learning with dynamic low-rank adaptation (Fed-DLoRA), which is combined with low-rank adaptation (LoRA) to effectively reduce parameters and communication costs while enhancing training efficiency.
The convergence analysis of Fed-DLoRA is conducted through stochastic gradient descent optimization coupled with singular value decomposition.
This analysis establishes the theoretical relationships among LoRA rank, vehicular scheduling strategies and the model's convergence characteristics.
Building on these insights, we formulate a joint optimization problem aimed at maximizing system performance.
To address this problem, we propose an adaptive rank, bandwidth and vehicle selection (ARBVS) algorithm that integrates enumeration with greedy optimization strategies.
The algorithm provides efficient rank selection and resource scheduling strategies for each FL communication round, thereby achieving effective performance improvements for the FL system.
Experimental results demonstrate that Fed-DLoRA achieves superior performance compared to conventional federated learning approaches, exhibiting enhanced accuracy, faster convergence, and improved communication efficiency.
\end{abstract}


\section{Introduction}
\IEEEPARstart{F}{ueled} by rapid advancements in intelligent transportation systems (ITS), artificial intelligence (AI), and next-generation mobile communication technologies, intelligent connected vehicles (ICVs) are evolving toward higher-order intelligence \cite{prathiba2021cybertwin}. 
These systems are now capable of handling complex tasks, including lane and obstacle detection \cite{shi2024lane}, bird's-eye view (BEV) perception \cite{zhao2024bev}, and trajectory prediction \cite{zhu2022cross}. 
Realizing these functionalities heavily relies on frequent machine learning model training. 
Federated learning (FL), a distributed machine learning framework, effectively mitigates the communication overhead and privacy leakage \cite{zhang2023dflnet} associated with the direct exchange of user data while capitalizing on the computational capabilities of end devices \cite{khan2021federated}. 

Although FL demonstrates significant potential in distributed machine learning, it faces critical challenges. 
At the communication level, frequent model parameter exchanges between clients and servers generate substantial overhead. 
At the computational level, clients' limited processing power restricts their ability to handle intensive local training while managing energy consumption effectively.
These challenges have motivated researchers to explore the use of model compression techniques to optimize the FL training process. 
Recent studies have integrated LoRA with federated learning to address the computational and communication challenges in wireless networks \cite{sun2024federated, zou2025joint}. 
However, these works primarily assume static or quasi-static network topologies where device participation and channel conditions remain relatively stable across communication rounds. 
In contrast, dynamic vehicular networks exhibit continuous vehicle entry/exit from coverage, time-varying channel conditions, and dynamic device participation per round—challenges that fundamentally differ from the stable network assumptions in existing federated LoRA frameworks. 
These dynamics create coupling effects among rank selection, bandwidth allocation, and vehicle participation that necessitate per-round adaptive strategies not addressed in prior work.

In ITS contexts, limited communication and computing resources prevent full vehicle participation in training \cite{quan2024federated}. 
Given ICVs' high mobility \cite{feng2022mobility} and real-time requirements, vehicles must complete model downloading, local training, and parameter uploading within base station coverage, imposing strict temporal constraints.  
Although dynamic bandwidth resource allocation can adjust uplink transmission rates, achieving optimal resource scheduling remains a significant challenge.
Recent works address joint optimization of compression rates, CPU frequency, and bandwidth allocation \cite{chen2023efficient}, spectrum sharing in V2V/V2I communications \cite{liang2019spectrum}, and client selection with bandwidth allocation considering mobility and data heterogeneity \cite{tang2025joint}.  
However, dynamically selecting appropriate ICVs and efficiently allocating resources during FL training with compression techniques remains critical for system performance optimization.

In light of the challenges discussed above, this paper proposes a lightweight FL algorithm named federated learning with dynamic low-rank adaptation (Fed-DLoRA), which employs LoRA technique to optimize FL training. 
Fed-DLoRA reduces local training and uplink transmission delays by decreasing training parameters and communication overhead, enabling greater ICV participation.  
Moreover, to address the specific optimization challenges posed by Fed-DLoRA, we propose an enumeration-based greedy algorithm for adaptive rank selection, ICV selection, and bandwidth allocation.
The contributions of this paper are summarized as follows:
\begin{itemize}
	\item The Fed-DLoRA algorithm is proposed, which significantly reduces the local training parameter size and uplink communication burden of ICVs in the FL system by integrating LoRA into the FL framework, thus increasing the participation rate of ICVs and effectively improving the learning accuracy and convergence speed of AI models.
	\item An analytical expression for the global average gradient of Fed-DLoRA is derived, and the difference in gradient between the central and local models is approximated and quantified using the singular value decomposition, which further reveals the intrinsic connection between the average gradient and the ICV selection strategy and LoRA parameters.
        \item Based on the theoretical derivation of the second point, a multivariate joint optimisation problem is constructed. We propose the ARBVS algorithm for dynamic rank, bandwidth, and vehicle selection optimization with low computational complexity.
\end{itemize}

\section{RELATED WORK}
\label{relatedWork}
\subsection{Communication Efficient Federated Learning}
Model compression techniques achieve lightweighting of AI models by reducing the number of parameters. 
Numerous studies have effectively applied these techniques to FL systems to enhance communication efficiency.
Common model compression methods in FL include pruning, sparsification, and LoRA.

Pruning and sparsification lighten models by removing redundant parameters, typically implemented post-training. 
Jia et al. \cite{jia2024dapperfl} proposed model fusion pruning integrating local and central knowledge.
Jiang et al. \cite{jiang2022model} developed two-stage adaptive pruning for FL. 
Chang et al. \cite{chang2024efficient} adaptively set pruning rates based on parameter importance. 
Guo et al. \cite{guo2024selective} used low-rank matrix decomposition to compress gradient information.

Unlike conventional compression, LoRA is a parameter-efficient fine-tuning strategy applied during training. 
Originally developed for large language models (LLMs), it maintains frozen pre-trained parameters while introducing trainable low-rank adaptation modules. 
Recent studies have demonstrated that this low-rank adaptation is not restricted to large transformer models, and can also be effectively integrated into convolutional neural network (CNN) architectures under appropriate settings. 
Aleem et al. \cite{aleem2024convlora} proposed ConvLoRA for convolutional neural networks, showing that low-rank adaptation can effectively be extended to CNN architectures.
Researchers have begun to apply LoRA to FL systems to improve efficiency. 
Liu et al. \cite{liu2025differentially} combined LoRA with FL and enhanced privacy preservation by adding random noise to the transmitted low-rank parameters.
Wu et al. \cite{wu2024fedfmsl} proposed a sparse-activated LoRA mechanism that flexibly adapts to the heterogeneous resource conditions of different edge devices through a controller module that adaptively decides where to insert the low-rank matrices in the network.
Su et al. \cite{su2024fedra} designed a LoRA FL architecture using a random assignment strategy, where a portion of the parameters of the LoRA module are randomly selected from the clients for aggregation in the model aggregation phase. 
Guo et al. \cite{guo2024selective} proposed a FL model in which only a specific one of the matrices in the LoRA decomposition is aggregated at the server side, which can further significantly reduce the communication overhead.

Sun et al. \cite{sun2024federated} proposed AirFL-LoRA integrating LoRA with over-the-air computation for federated learning.
Zou et al. \cite{zou2025joint} introduced ROBA for joint rank and bandwidth optimization in heterogeneous federated LoRA fine-tuning.

Furthermore, integrating model compression techniques within the FL framework can significantly affect the model's convergence performance. 
Prior studies have investigated convergence analysis of FL systems with various compression methods \cite{feng2022mobility, wu2024fedfmsl}, with some works explicitly addressing dynamic vehicular scenarios \cite{chang2024efficient}. 
However, convergence analysis specifically for LoRA-enhanced FL is limited. While recent works have incorporated LoRA into federated learning frameworks \cite{sun2024federated}, these studies primarily focus on static or quasi-static network environments without explicitly modeling high-speed vehicle mobility. 
Convergence analysis and system optimization for LoRA-based FL in dynamic vehicular scenarios remains largely unexplored. Consequently, this study emphasizes an in-depth examination of model convergence and system optimization in LoRA-enhanced dynamic FL systems.

\begin{figure*}[ht]
    \centering
    \includegraphics[width=0.7\linewidth]{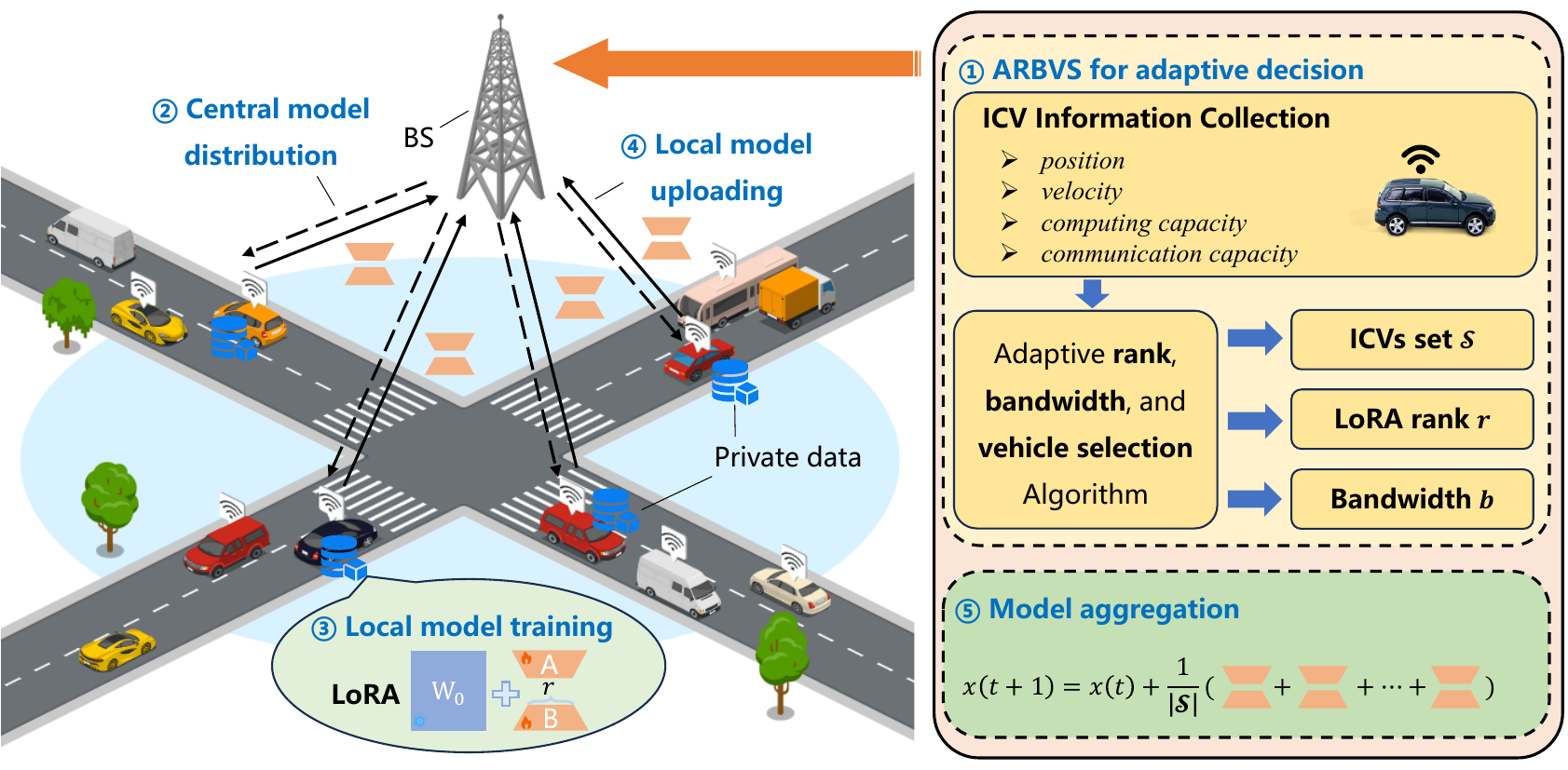}
    \caption{The framework of Fed-DLoRA.}
    \label{fig:framework of Fed-DLoRA}
\end{figure*}
\subsection{Client Selection and Bandwidth Allocation for FL}
Beyond model compression, efficient communication resource allocation is crucial in FL optimization.
Bandwidth allocation optimization is widely studied.
Yang et al. \cite{9264742} proposed iterative algorithms for optimal time, bandwidth, power, and computational frequency allocation. 
Song et al. \cite{song2023distributed} developed distributed reinforcement learning-based resource allocation maximizing spectrum and energy efficiency. 
Lei et al. \cite{lei2024energy} proposed an iterative algorithm based on a lightweight FL architecture, the overall energy consumption of the FL system is further reduced by fine-grained resource allocation and learning parameter tuning.

It is also a common idea to consider the client's adaptive choice while allocating resources.
Shi et al. \cite{shi2020joint} allocated more bandwidth to devices with poor channels and weak computation, selecting devices with minimal update time. 
Chen et al. \cite{chen2023efficient} devised a co-optimisation scheme to co-optimise the compression rate, CPU frequency, uplink transmission power, and bandwidth allocation of the devices. 
Chang et al. \cite{chang2024efficient} integrated vehicle selection, model pruning, and bandwidth allocation for IoV scenarios.  
Zhang et al. \cite{zhang2023joint} proposed deep reinforcement learning-based joint device scheduling and bandwidth allocation.

Despite significant progress in FL, its deployment in IoV renders performance highly sensitive to bandwidth, computational resource management, and client selection policies.
Consequently, optimizing these critical factors within specific network settings remains an urgent challenge.

\section{FL WITH LoRA FOR IoV}
\label{methodFramework}
\subsection{System Setup}
\label{scence}
To simulate realistic traffic conditions, this study constructs a dynamic vehicular networking scenario with a single base station, see Fig. \ref{fig:framework of Fed-DLoRA}. 
The scenario contains a fixed base station and several mobile ICVs. 
The signal coverage radius of the base station is $R$. Given the mobile nature of ICVs, ICVs will frequently enter and exit the coverage area of the base station throughout the training process of FL, which may result in some ICVs located within the coverage area also failing to participate in the FL process. 
At a specific moment $t$, the set of ICVs in the base station coverage area is denoted as $\mathcal{U} = \{u_1, u_2, ... , u_n\}$, and the total number of ICVs is $|\mathcal{U}|$. 
where the speed of the $n$-th ICV is denoted as $v_n$, its local dataset is $D_n$. 
Each ICV is equipped with an onboard unit that supports communication and local learning. 
The ICVs utilize cellular vehicle-to-everything (C-V2X) communication technology and frequency division multiple access (FDMA) for both uplink and downlink communications with the base station. 
Communications are performed via a single-hop link between the ICV and the base station, with the reliability of data transmission ensured by vehicle-to-infrastructure (V2I) technology. 
Direct communication between ICVs is not considered.
\subsection{Framework of Fed-DLoRA}
This subsection focuses on the key steps of the Fed-DLoRA architecture.
The FL training process performs a total of $T$ rounds of FL.
For communication round $t$, Fed-DLoRA performs the following steps:

1) \textit{ARBVS for adaptive decision}: Each ICV is assumed to participate actively in FL training. BS immediately establishes connections with all ICVs, collects speed and position information, and determines an adaptive ICV set $\mathcal{S}$, the corresponding bandwidth allocation $b$, and the LoRA rank $r$.

2) \textit{Central model distribution}: BS sends the current central model $x(t)$ to the ICVs in set $\mathcal{S}$ through the downlink channel.

3) \textit{Local model training}: The selected ICV set $\mathcal{S}$ constructs the local model $x_n(t)$ based on the rank of the optimal LoRA. 
Subsequently employs local sub-dataset $D_n$ for local training to obtain the updated local model $x_n(t+1)$. 
The local model comprises two components: $W_0$  and the LoRA matrix. 
During training, $W_0$ remains frozen and is used solely for inference, while the two low-rank matrices constituting the LoRA undergo gradient descent updates.
The update strategy of the local model is defined as:
\begin{equation}
    \begin{aligned}
        x_n(t+1) &= W_0(t) + B_n(t+1)A_n(t+1),
    \end{aligned}
    \label{eq:fl}
\end{equation}where $W_0(t)$ represents the global model parameters distributed by the server in round $t$. Note that $W_0(t)$ serves as the initial state of the local model $x_n(t)$. $B_n(t+1)$ and $A_n(t+1)$ are the LoRA training matrices of the local model of the $n$-th ICV at time $t+1$.
Based on gradient descent, the local model update formula is:
\begin{equation}
    \begin{aligned}
        x_n(t+1) = x_n(t) - \eta \nabla \mathcal{L}_n(x_n(t)),
    \end{aligned}
    \label{eq:localup}
\end{equation}where $\eta$ is the learning rate and $\mathcal{L}_n(\cdot)$ is the loss function of the $n$-th ICV. 

4) \textit{Local model uploading}: All ICVs in the set $\mathcal{S}$ upload their local models to the BS via a wireless link. 
Unlike conventional FL, only the low-rank matrices from the LoRA method are transmitted during this process. 
Compared to transmitting all parameters of the original model, the LoRA approach, markedly reduces the number of collaborative parameters and significantly decreases upload communication latency.

5) \textit{Model aggregation}: The collected local models are aggregated at the BS. 
The aggregation is performed using uniform aggregation with the following procedure:
\begin{equation}
    \begin{aligned}
        x(t+1) = x(t) - \eta \sum_{n=1}^{|\mathcal{S}|} \alpha_n g_n(x_n(t)),
    \end{aligned}
    \label{eq:agr}
\end{equation}where $\alpha_n$ is the weight of the $n$-th ICV. $x(t+1)$ and $x(t)$ depend only on the local updates uploaded by the ICVs.

Repeat the above steps over and over again until the designated communication round.
Fed-DLoRA aims to minimise the global loss function minimisation:
\begin{equation}
        \min_{x} \mathcal{L}(x) := \sum_{n \in \mathcal{S}}\alpha_n \mathcal{L}_n(x), 
\end{equation}where $\mathcal{S}$ is the set of ICVs participating in the FL process, which is adaptively determined by the optimisation algorithm, $\alpha_n$ is the weighting factor for aggregation correlation.
Fed-DLoRA adopts the LoRA to lighten the FL training model, aiming to shorten the latency of the local model training and uploading process and to reduce the bandwidth requirement of ICVs. 
This will enable more ICVs to participate in the FL training process, which in turn accelerates the convergence of the FL model while improving its stability. 
For more details on the implementation of LoRA, see the subsequent sections.
\subsection{FL with LoRA}
LoRA is a low-resource training technique initially introduced for fine-tuning large models \cite{hu2022lora}, applicable to various neural network architectures. 
Each layer decomposes into two low-rank matrices, simulating forward propagation as their product, as shown at Fig. \ref{fig:LoRA}.
\begin{figure}[t]
    \centering
    \includegraphics[width=0.45\linewidth]{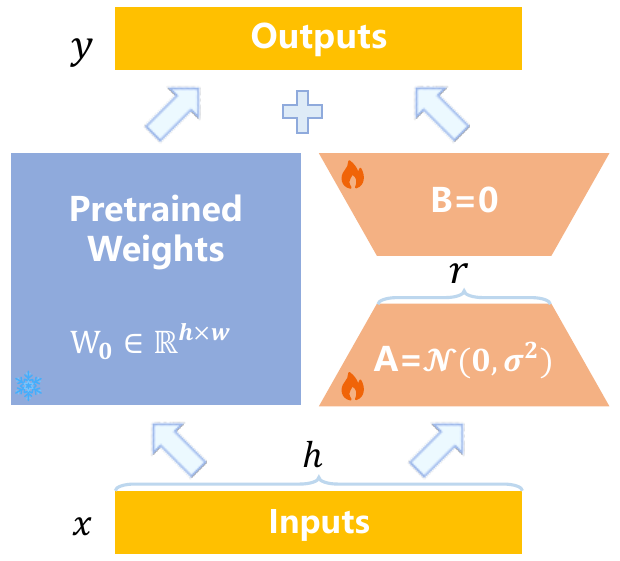}
    \caption{Schematic diagram of LoRA structure.}
    \label{fig:LoRA}
\end{figure}

LoRA decouples the model training process into a residual architecture, and constructs the following form of low-rank decomposition for the current layer's weight matrix $W_0 \in \mathbb{R}^{h \times w}$:
\begin{equation}
    \begin{aligned}
        W = W_0 + \Delta W = W_0 + BA ,
    \end{aligned}
    \label{eq:LoRA}
\end{equation}where $B \in \mathbb{R}^{h \times r}$ and $A \in \mathbb{R}^{r \times w}$ are LoRA matrices, $r$ is the rank of LoRA, and $r << \min(h, w)$. 
The training process freezes the parameter $W_0$ and trains only the parameters in matrices $B$ and $A$. 
The forward propagation process becomes:
\begin{equation}
    y = W_0x + BAx.
    \label{eq:LoRA_forward}
\end{equation}
In this way, the weight matrix involved in backpropagation is transformed from the original single full-rank matrix to two low-rank matrices during the local training process. 
Similarly, only these two low-rank matrices need to be transmitted during the model upload phase. 
Take a single weight matrix $W^{h \times w}$ as an example for illustration: the number of training and transmission parameters for the original full-rank matrix is $h \times w$, where as the number of parameters required is reduced to $r(h + w)$ after applying LoRA. 
This reduces the amount of data to be trained and transmitted, thus reducing collaboration latency.

LoRA reduces the number of trainable and transmitted parameters, which can lower local computation and communication overhead. This creates more favorable conditions for increased ICV participation and can improve the convergence behavior of the FL process under the considered IoV settings. It is important to note that LoRA is utilized here as a lightweight parameterization mechanism for federated training, rather than to claim it inherently replicates the same convergence behavior as in large-model fine-tuning scenarios.
The LoRA rank significantly influences training, uploading delays and FL convergence, necessitating adaptive rank selection through optimization algorithms detailed in subsequent sections.

\section{Optimisation problem analysis for Fed-DLoRA}
\label{analysis}
\subsection{Mobility Limitations}
In the Fed-DLoRA architecture, ICVs participating in federated learning are required to complete all operations within the BS coverage. 
Consequently, the mobility of ICVs and the latency associated with computational tasks must be analyzed independently.

a) \textit{ ICV mobility delay}: As described in \cite{pervej2023resource}, we simulated an ICV mobility scenario. 
Based on each ICV’s position and current speed, the maximum sojourn time limit within the base station coverage is computed as follows:
\begin{equation}
    T_n^{st} = \frac{L_n}{v_n},
    \label{eq:st}
\end{equation}where $L_n$ is the distance of the $n$-th ICV from the BS boundary. 
$L_n$ is determined based on the current position of the $n$-th ICV relative to the BS coverage.

b)\textit{ Computational task latency}: Consistent with the work described in \cite{chang2024efficient}, the latency arises from two primary stages: local model training and local model uploading. 
The latency encountered during local model training is predominantly influenced by the number of training iterations, the volume of local data, and the computing capacity of the ICV. 
The local training delay for the $n$-th ICV is defined as follows:
\begin{equation}
    t_n^l = \frac{\gamma E D_n \bar{w}_n r N_{LoRA}}{f_n N},
    \label{eq:t_n^l}
\end{equation}where $E$ is number of FL training epoch per communication round, $D_n$ is the dataset size of the $n$-th ICV, $\bar{w}_n$ is the number of CPU rounds required by the ICV to process a sample, $r$ is the rank of LoRA, $N$ is the number of parameters in the full-rank ResNet18, and $f_n$ is the computational power of the $n$-th ICV, $\gamma$ is the compensation coefficient. $\gamma$ is used to adjust the computational delay incurred by the base model that performs only forward reasoning in the training phase and does not participate in parameter updating. Its value fluctuates within the range $[1.3, 1.5]$.

The local model uploading delay is affected by the channel transmission rate and the number of model parameters. 
For the $n$-th ICV, the uplink transmission delay is:
\begin{equation}
    t_n^u = \frac{d r N_{LoRA}}{C_n^u},
    \label{eq:t_n^u}
\end{equation}where $N_{LoRA}$ is the number of parameters of LoRA,  $d$ is the model bit width, and $C_n^u$ is the transmission rate of the $n$-th ICV. 
Since FDMA is used in this research, the $n$-th ICV has for $C_n^u$ then:
\begin{equation}
    C_n^u = b_n \log_2(1 + \frac{P_n h_n}{b_n N_0}),
    \label{eq:C_n^u}
\end{equation}where $b_n$ is the bandwidth allocated to the $n$-th ICV, $P_n$ is the transmit power, $h_n$ is the channel gain, and $N_0$ is the power spectral density of Gaussian noise.

\subsection{Convergence Analysis for the Fed-DLoRA Algorithm}
This subsection analyses the convergence of Fed-DLoRA. 
Common assumptions are made based on this work\cite{chang2024efficient, wu2024fedfmsl}.

\textit{Assumption 1(smoothness)}: The locally non-convex loss function of the ICV is $\beta-smooth$:
\begin{equation}
    \left\| {\nabla \mathcal{L}({{\rm{w}}_{\rm{1}}}) - \nabla \mathcal{L}({{\rm{w}}_{\rm{2}}})} \right\| \le \beta \left \| {\rm{w}_{\rm{1}}} - {\rm{w}_{\rm{2}}} \right \|, \forall{{\rm{w}_{\rm{1}}}, {\rm{w}_{\rm{2}}}},
    \label{eq:assumption1}
\end{equation}where $\beta$ is constant and $\nabla \mathcal{L}(\cdot)$ is the gradient of $\mathcal{L}(\cdot)$. $\rm{w}_1$ and $\rm{w}_2$ are the model weights for any two different ICVs.

\textit{Assumption 2(unbiased gradient)}: Unbiased gradient expects that the gradient computed by the client in small batches is equal to the gradient obtained by its loss, i.e., the gradient estimate is accurate.
\begin{equation}
    \mathbb{E}_{\xi|D_n}[g_n(\rm{w})] = \nabla \mathcal{L}_n(\rm{w}), \forall n, \rm{w},
    \label{eq:assumption2}
\end{equation}where $\xi$ is the data of any random batch, $g_n(\cdot)$ is the stochastic gradient of the $n$-th ICV, and $\nabla \mathcal{L}_n(\cdot)$ is the gradient of the $n$-th ICV model.

\textit{Assumption 3(bounded variance)}: Bounded variance requires that the entire gradient exhibits limited fluctuations. 
Although the expectation of $g_n$ is constrained to equal $\nabla \mathcal{L}_n$ through an unbiased gradient, this unbiasedness applies only in the average sense, meaning that individual gradients may still be biased.
\begin{equation}
    \mathbb{E} \left [ \left \| g_n({\rm{w}}) - \nabla \mathcal{L}_n({\rm{w}}) \right \|^2 \right ] \le \sigma^2, \forall{n, \rm{w}}.
    \label{eq:assumption3}
\end{equation}

\textit{Assumption 4(bounded gradient)}: The singular values of the full-rank gradient matrix of the central model are bounded by $M$. 
Here, $\sigma_{l, i}$ denotes the i-th singular value of the full-rank gradient at layer l, while $k_l = min(h_l, w_l)$ specifies the number of singular values at that layer, reflecting the matrix's full-rank dimensionality.
\begin{equation}
    M \ge \sigma_{l, 1} \ge \sigma_{l, 2} \ge \dots \ge \sigma_{l, k_l}, \forall l,
    \label{}
\end{equation}where $\sigma_{l, n}$ is the $n$-th singular value of the $l$-th layer of the model.

For the central model at time $t+1$ , according to the FL central model updating process \eqref{eq:localup} and the smoothing assumption \eqref{eq:assumption1}, then:
\begin{equation}
    \small
    \begin{aligned}
    \mathbb{E}[\mathcal{L}(x(t+1))] 
    & = \mathbb{E}\left[ \mathcal{L}(x(t) - \eta \sum_{n=1}^{|\mathcal{S}|} \alpha_n g_n(x_n(t))) \right] \\
    & \le \underbrace{\mathbb{E}\left[ \mathcal{L}(x(t)) \right]}_{C_1} \\
    & + \underbrace{\mathbb{E}\left[ \left< \nabla \mathcal{L}(x(t)), -\eta \sum_{n=1}^{|\mathcal{S}|} \alpha_n g_n(x_n(t))) \right> \right]}_{C_2} \\
    & + \underbrace{\frac{\beta}{2} \mathbb{E}\left[ \left \| -\eta \sum_{n=1}^{|\mathcal{S}|} \alpha_n g_n(x_n(t))) \right \|^2 \right]}_{C_3}.
    \end{aligned}
    \label{eq:fl2C1C2C3}
\end{equation}

\newtheorem{theorem}{\bf Theorem}
\begin{theorem}
\label{thm:1}
Under Assumptions 1-3, given the central model updating process \eqref{eq:agr}, the expected loss of the central model at time $t+1$ is upper-bounded by:
\begin{equation}
    \begin{aligned}
    \mathbb{E}[\mathcal{L}(x(t+1))] 
    &\le \mathbb{E}[\mathcal{L}(x(t))] - \frac{\eta}{2} \mathbb{E} \left[ \|\nabla \mathcal{L}(x(t))\|^2 \right] \\
    &\quad + \frac{\eta^2 \beta}{2} \sigma^2 \sum_{n=1}^{|\mathcal{S}|} \alpha_n^2 \\
    &\quad + \frac{\eta}{2} \sum_{n=1}^{|\mathcal{S}|} \alpha_n \mathbb{E} \left[ \| \nabla \mathcal{L}(x(t)) - \nabla \mathcal{L}_n(x_n(t)) \|^2 \right].
    \end{aligned}
    \label{eq:theorem_result}
\end{equation}
\end{theorem}

\begin{IEEEproof}[Proof of Theorem \ref{thm:1}]
    The proof is given in \textbf{Appendix \ref{appendix:proof_thm1}}.
    \renewcommand{\IEEEQED}{}
\end{IEEEproof}
In (\ref{eq:theorem_result}), $\nabla \mathcal{L}(x(t))$ is the gradient of the central model, and $\nabla \mathcal{L}_n(x_n(t))$ is the gradient of the $n$-th ICV local model. $\nabla \mathcal{L}_n(x_n(t))$ corresponds to the gradients of $B_n(t)$ and $A_n(t)$ at time $t$. The difference between the two of them cannot be obtained by precise numerical calculation. Therefore, other computational methods need to be explored to represent the difference between the central model gradient and the local model gradient.

\subsection{Singular Value Decompositiont for Gradient Gap}
\label{objFun}
From a matrix perspective, LoRA employs a low-rank matrix to approximate a full-rank matrix during fine-tuning. 
Rather than directly computing the complete gradient, LoRA optimizes the gradients with respect to matrices $B$ and $A$ (i.e., $\nabla_B \mathcal{L}$ and $\nabla_A \mathcal{L}$) and subsequently updates W by means of $\Delta W = BA$. 
Hao et al. \cite{hao2024flora} pointed out that LoRA projects the full-rank gradient onto a lower-dimensional subspace, thereby constraining the gradient update’s degrees of freedom via low-rank decomposition and effectively mapping high-dimensional weight adjustments into a lower-dimensional space. 
Essentially, the term $\nabla \mathcal{L}(x(t)) - \nabla \mathcal{L}_n(x_n(t))$ captures the projection (approximation) error between the native full‑rank gradient and its low‑rank counterpart induced by LoRA, rather than a mere difference in vector dimensionality. 
This low‑rank constraint on the LoRA gradient can be quantitatively characterized using matrix decomposition techniques such as singular value decomposition (SVD).

For layer $l$, the SVD of the full gradient $G_l$ is assumed to be:
\begin{equation}
    G_l = U_l \Sigma_l V_l^T,
    \label{eq:svd1}
\end{equation}where $\Sigma_l$ is the diagonal matrix containing the singular values of the weight matrix $\{\sigma_{l,1}, \sigma_{l,2}, \cdots, \sigma_{l, k_l}\}$. LoRA Gradient The gradient of the low-rank approximation of $G_{l, LoRA}$ retains the largest r singular values, and can be expressed as follows:
\begin{equation}
    G_{l, LoRA} = \sum_{i=1}^r \sigma_{l, i} u_{l, i} v_{l, i}^T.
    \label{eq:svd2}
\end{equation}

Then the difference $\Delta G_l$ corresponds to the remainder:
\begin{equation}
    \Delta G_l = G_l - G_{l, LoRA} = \sum_{i=r+1}^{k_l} \sigma_{l, i} u_{l, i} v_{l, i}^T.
    \label{eq;svd3}
\end{equation}where $k_l=\min(h_l, w_l)$ and $h$ and $w$ are the weight dimensions of the model $l$.

The monolayer difference is expressed in terms of the Frobenius norm, which is obtained according to \textit{Assumption 4}:
\begin{equation}
    \left \| \Delta G_l \right \|_F = \sqrt{\sum_{i=r+1}^{k_l} \sigma_{l, i}^2}
    \le M \sqrt{k_l - r}.
    \label{eq:svd4}
\end{equation}

Then the difference paradigm of the whole model is:
\begin{equation}
    \begin{aligned}
        \left \| \Delta G_{total} \right \|_F
        = \sqrt{\sum_{l=1}^L \left \| \Delta G_l\right \|_F^2} 
        \le M\sqrt{K-Lr},
    \end{aligned}
    \label{eq:svd5}
\end{equation}where $K=\sum_{l=1}^L k_l$ denotes the total number of singular values of all weight matrices in the model, and $Lr$ denotes the sum of the LoRA ranks of all layers.
Substituting \eqref{eq:svd5} into \eqref{eq:theorem_result} yields:
\begin{equation}
    \begin{aligned}
        \mathbb{E}[\mathcal{L}(x(t+1))] 
        &\le \mathbb{E}[\mathcal{L}(x(t))] - \frac{\eta}{2} \mathbb{E} \left[ \|\nabla \mathcal{L}(x(t))\|^2 \right] \\
        &\quad + \frac{\eta^2 \beta}{2} \sigma^2 \sum_{n=1}^{|\mathcal{S}|} \alpha_n^2 + \frac{\eta}{2} M^2\sum_{n=1}^{|\mathcal{S}|} \alpha_n (K-Lr).
    \end{aligned}
    \label{eq:g1}
\end{equation}

The upper bound on the gradient of the FL process at the $\tau_t$ moment can be obtained by shifting the term:
\begin{equation}
    \begin{aligned}
        \mathbb{E} \left [ \left \| \nabla \mathcal{L}(x(t)) \right \|^2 \right]
        \le \frac{2}{\eta} \left ( \mathbb{E}\left[ \mathcal{L}(x(t)) \right ] - \mathbb{E}[\mathcal{L}(x(t+1))] \right ) \\
        + \eta \beta \sigma^2\sum_{n=1}^{|\mathcal{S}|} \alpha_n^2
        + M^2\sum_{n=1}^{|\mathcal{S}|} \alpha_n (K-Lr).
    \end{aligned}
    \label{eq:g2}
\end{equation}

Summing inequality \eqref{eq:g2} over $t=1, \dots, T$ and dividing both sides by $T$ yield:
\begin{equation}
    \begin{aligned}
        \frac{1}{T}\sum_{t=1}^T \mathbb{E}\left [ \left \| \nabla \mathcal{L}(x(t))\right \|^2 \right ]
        \le 
        \frac{2}{\eta T} \left ( \mathbb{E}\left[ \mathcal{L}(x(1)) \right ] - \mathcal{L}^* \right ) \\
        + \eta \beta \sigma^2\sum_{n=1}^{|\mathcal{S}|} \alpha_n^2
        + M^2\sum_{n=1}^{|\mathcal{S}|} \alpha_n (K-Lr),
    \end{aligned}
\end{equation}where we want the central model to converge to the desired value $\mathcal{L}^*$ , $\mathcal{L}^* := min_{x}\mathcal{L}(x)$ at time $T+1$. 
In this paper we consider subdatasets of the same size for each ICV that are available, then we have:
\begin{equation}
        \begin{aligned}
            \frac{1}{T}\sum_{t=1}^{T}\mathbb{E}\left [ \left \| \nabla \mathcal{L}(x(t)) \right \|^2 \right]
            \le
            \frac{2}{\eta T} \left ( \mathbb{E}[\mathcal{L}(x(1)) - \mathcal{L}^*] \right )
            + \frac{\eta \beta \sigma^2}{|\mathcal{S}|} \\
            + M^2(K-Lr).
        \end{aligned}
\end{equation}

This demonstrates that the learning performance of Fed-DLoRA is closely related to both the ICV selection strategy and the LoRA rank. 
Consequently, determining how to regulate these influencing factors to optimize FL system performance is a pressing challenge. 
In the subsequent sections, we formulate a joint optimization problem and propose a low-time-complexity algorithm aimed at maximizing the overall performance of the FL system.

\section{JOINT OPTIMIZATION OF RANK, BANDWIDTH AND ICV SELECTION}
\label{optProblem}
In this section, we introduce an adaptive algorithm designed to address the joint optimization of LoRA rank, bandwidth allocation, and ICV selection in Fed-DLoRA systems. 
The algorithm employs a greedy strategy to efficiently minimize the objective function while adhering to strict delay and bandwidth constraints, making it particularly well-suited to the dynamic and resource-constrained environment characteristic of ICVs.
\subsection{Problem Formulation}
We consider a FL optimization process in an IoV scenario using FDMA. BS must adaptively coordinate all ICVs within its coverage area at time $t$. 
The objective is to make adaptive Fed-DLoRA decisions by selecting a subset $S$ of ICVs, determining the rank $r$ of the Fed-DLoRA system, and assigning communication bandwidth $b_n$  to each ICV in subset $S$.
The goal is to minimise the following function:
\begin{subequations}
    \label{eq:optimization}
        \begin{align}
        \mathcal{P}: \min_{\{b_n, \mu_n, r\}} & \quad \frac{2}{\eta T} \left ( \mathbb{E}[\mathcal{L}(x(1)) - \mathcal{L}^*] \right ) + \frac{\eta \beta \sigma^2}{|\mathcal{S}|} + M^2(K - Lr) \label{optimization:obj} \\
        \text{s.t.} \quad & C_1: \quad \mu_n \in \{0, 1\}, \quad \forall n = 1,\dots,\mathcal{U}, \label{optimization:c1} \\
        & C_2: \quad \sum_{n=1}^{|\mathcal{U}|} \mu_n = |\mathcal{S}|, \label{optimization:c2} \\
        & C_3: \quad \sum_{n=1}^{|\mathcal{U}|} \mu_n b_n \leq B ,\label{optimization:c3} \\
        & C_4: \quad 0 \leq b_n \leq B, \quad \forall n \in \mathcal{S}, \label{optimization:c4} \\
        & C_5: \quad t_n^l + t_n^u \le T_n^{st} \le \frac{L_n}{v_n}, \quad \forall n \in \mathcal{S}, \label{optimization:c5}
        \end{align}
\end{subequations}where $\mu_n$ indicates whether the ICV is selected or not, and $1$ if it is selected. 
$T_n$ is the total latency of $n$-th, including the local training latency $t_n^l$ and the model uplink transmission latency $t_n^u$.  
The total bandwidth $B$ limits the communication bandwidth allocated to each ICV.

In this study, we analyze the objective function associated with the optimization problem. 
As shown in \eqref{optimization:obj}, the first term is a constant.
The second and third terms are linearly related to $\frac{1}{|S|}$ and $r$, respectively. 
Constraints \eqref{optimization:c3}–\eqref{optimization:c5} ensure that the training and upload delays do not exceed the dwell time, the ICV selection variable remains binary, and the allocated bandwidth does not exceed the total available capacity.
Our objective is to maximize the number of ICVs selected at time $t$ while simultaneously achieving the highest possible rank value $r$. 
In this process, we concurrently optimize the three variables $\{b_n, \mu_n, r\}$. 
Next, we analyze the constraints and variables inherent in the optimization problem. 
Specifically, $b_n$ and $r$ are continuous variables, whereas $\mu_n$ is a binary variable, and both the objective function and the constraints are nonlinear. 
Moreover, the interaction between $b_n$ and r jointly influences the delay $T_n$. 
This combination of factors results in a complex nonconvex mixed-integer nonlinear programming problem, making the direct derivation of a global optimal solution particularly challenging.

\subsection{Adaptive Rank, Bandwidth and ICV Selection Algorithm}
To address this mixed integer nonlinear programming problem, we propose an adaptive rank, bandwidth and vehicle selection (ARBVS) algorithm to jointly optimize $\{b_n, \mu_n, r\}$. 
ARBVS employs a greedy strategy to efficiently minimize the objective function while rigorously adhering to constraints on delay and bandwidth, rendering it particularly effective for the dynamic, resource-constrained environments typical of ICVs.

There is an implicit constraint within LoRA .
LoRA is introduced to reduce the size of the FL training model and to decrease communication overhead. 
The LoRA rank $r$ significantly influences the size of the FL model. It is essential to ensure that the LoRA model does not exceed the size of the original model, as this would cause the approach to degenerate into standard FedAvg. 
Consequently, an upper bound on the rank, denoted as $\mathcal{R}$. 
When disregarding bias weights, $r$ is linearly related to the LoRA model and can be readily computed from the number of parameters in both the original and the LoRA models. Moreover, the lower bound for $r$ is $1$, and it must be an integer, we have $r \in [1, \mathcal{R}]$. 
Thus, by enumerating over each possible value of $r$, the objective function’s value can be determined for each case, with the minimum value corresponding to the optimal solution. 
The optimization problem $\mathcal{P}$ becomes:
\begin{equation}
    \label{eq:optimization_prime}
    \begin{aligned}
    \mathcal{P'}: \min_{\{b_n, \mu_n\}} & \quad \frac{\eta \beta \sigma^2}{|\mathcal{S}|} + M^2(K - Lr) \\
    \text{s.t.} \quad & \eqref{optimization:c1} - \eqref{optimization:c5}.
    \end{aligned}
\end{equation}

Since $r$ is given in each case and only $b_n$ is a variable in \eqref{optimization:c5} and \eqref{eq:st} - \eqref{eq:C_n^u}, the expression associated with $b_n$ reduces to:
\begin{equation}
    b_n \log_2(1 + \frac{P_n h_n}{b_n N_0})
    \ge 
    C_n^u 
    \geq 
    \frac{d r N_{\text{LoRA}}}{T_n^{\text{st}} - \frac{\gamma E \bar{w}_n r  N_{\text{LoRA}}}{f_n N}},
    \label{eq:b_min}
\end{equation}

A lower bound can be derived by dichotomizing $b_n$, which denotes the minimum bandwidth $b_n^{min}$ required for each ICV to complete both training and uploading within the current region. 
Under a limited bandwidth constraint, the number of ICVs participating is maximized by allocating the minimum required bandwidth $b_n^{min}$ to each ICV.

The ICVs are selected using a greedy algorithm. 
First, all values of $b_n^{min}$ are sorted in ascending order using an efficient sorting algorithm. 
Next, ICVs are selected sequentially, beginning with the smallest bandwidth requirement, and their bandwidths are accumulated until the condition $\sum b_{|\mathcal{S}|} > B$ is met, at which point the selection process terminates. 
The number of ICVs selected at this stage represents the maximum achievable for the current rank $r$.

By enumerating the local optimums under different $r$-constraints, an effective allocation policy is obtained.

\subsection{Complexity Analyses}
The ARBVS's overall complexity primarily influenced by the enumeration of the rank $r$, the adaptive communication bandwidth allocation, and the adaptive ICV selection strategy. 
First, for enumerating the rank $r$ where $r \in [1, \mathcal{R}]$, the process incurs a complexity of $\mathcal{O}(\mathcal{R})$, with $\mathcal{R}$ determined by the ratio of the number of parameters in the native model to those in the LoRA model. 
Second, regarding the adaptive communication bandwidth allocation and adaptive ICV selection policy, the minimum communication bandwidth $b_n^{min}$ must be computed for each ICV in the current scenario. 
For each ICV, the minimum feasible bandwidth $b_n^{min}$ is obtained by solving a one-dimensional nonlinear equation via a numerical root-finding method, with a target accuracy $\epsilon$. 
The required number of iterations is $O(\log_2(1 / \epsilon))$, which is independent of $|U|$. Therefore, computing all $b_n^{min}$ has complexity $O(|U|\log_2(1 / \epsilon))$.
After that, ARBVS sorts all ICVs according to $b_n^{min}$, which costs $O(|U| \log_2{|U|})$, and performs a greedy selection in $O(|U|)$ time.
Combining these stages, the overall computational complexity of the algorithm is $\mathcal{O}(\mathcal{R}|U| \max (\log_2(1 / \epsilon), \log_2{|U|}))$, which grows essentially linearly with the number of ICVs and is therefore highly efficient for the considered IoV scenarios.

\section{EXPERIMENTS RESULTS}
\label{experiment}
We present experiments that assess the performance of the Fed-DLoRA architecture across multiple datasets in various environmental settings.

\subsection{Experimental settings}
1) \textit{Datasets}: Experiments use CIFAR-10 and CIFAR-100 datasets \cite{krizhevsky2009learning} from the official \textit{torchvision} library without preprocessing.

Following \cite{chang2024efficient}, we simulate a C-V2X-based single-base station ICV mobility scenario as described in Section \ref{scence}. 
Each ICV is characterized by parameters including position, speed, and computational capacity, which collectively constrain its participation in FL. 
In each FL round, a total of $20$ ICVs within the base station’s coverage area are engaged. Comparative experiments evaluate various FL algorithms using the CIFAR-10 and CIFAR-100 datasets under both independent and identically distributed (IID) and non-IID data distributions. For the IID, training data is randomly sampled and assigned to each ICV. For the non‑IID, we construct separate non‑IID partitions for CIFAR‑10 and CIFAR‑100. For CIFAR‑10, each ICV is randomly assigned training samples from $3$ randomly selected classes. For CIFAR‑100, each ICV is randomly assigned training samples from $30$ randomly selected classes. The test set remains unchanged.

2) \textit{Experimental model}: All experiments are conducted based on a scaled-down version of ResNet18\cite{he2016deep} with a channel basewidth of $32$. 
The full-rank ResNet18 includes a total of about $2.75$ million parameters (excluding bias terms).
LoRA modules are primarily applied to the convolutional layers of ResNet18, which account for the majority of model parameters. 
For implementation consistency within the unified framework, LoRA is also applied to the final fully connected layer, although its contribution to parameter reduction is relatively minor.
The primary performance benefits of our framework stem from the overall parameter-efficient updates and their synergy with dynamic resource scheduling, rather than the low-rank treatment of specific small components.
All models in our experiments are trained from scratch with random initialization, and no pre-trained weights are used.
This setting is adopted to ensure a fair comparison across all methods, as none of the baseline methods rely on pre-trained models.

3) \textit{Parameter setting}: All experiments are conducted on an identical hardware platform with consistent hyper-parameter configurations to ensure the reproducibility and fairness of the experimental results.

\begin{itemize}
    \item  \textit{Hardware setup}: All simulation experiments are completed on a deep learning server. The server is equipped with NVIDIA RTX 4090 GPU, Intel 14900KF CPU, the operating system is Ubuntu 20.04, and the programming environment is Python 3.10 and PyTorch 2.3.0.
    \item \textit{FL training hyperparameters}: Stochastic gradient descent (SGD) is chosen as the optimiser. Batch size is $32$, number of local training epochs $E$ is $4$. Loss function is cross entropy. learning rate $\eta$ is $0.01$.
    \item \textit{C-V2X communication parameters}: Communication-related parameters are primarily referenced in the literature \cite{pervej2023resource, chang2024efficient}. The simulation environment consists of a single-base station C-V2X scenario, where the base station is positioned at the center of the coverage area with a signal radius of $500$ m. The ICV speed, denoted by $v_n$, ranges from $12$ to $22$ m/s. The total communication bandwidth $B$ is $10$ MHz, and the ICV transmission power $P_n$ is set at $28$ dBm. The path loss is modeled as $128.1+37.6 \log_2(X)$, where $X$ represents the distance in kilometers between the ICV and the BS. Furthermore, the channel noise power $N_0$ is specified as $-174$ dBm.
    \item \textit{ICV arithmetic parameters}: For both IID and non-IID data distributions, each ICV is randomly assigned $3,000$ training samples. The CPU clock cycles required for processing each sample, denoted by $\bar{w}_n$, fall within the range of $[0.8, 1.2] \times 10^7$ cycles, and the CPU frequency of each ICV, represented by $f_n$, is set $[1.9, 3] \times 10^9$ Hz. $\mathcal{R}$ is set to $32$. $M$ is set to $0.1$.
\end{itemize}

4) \textit{Baseline method}:
Several representative FL algorithms are selected for comparative experiments.

\begin{itemize}
    \item \textit{FedAvg}: Clients upload local model parameters after training, and the server computes weighted averages to generate the updated global model.
    \item \textit{FedPT} \cite{sidahmed2021efficient}: Some model parameters are frozen throughout the FL training process to reduce the computational burden on the clients. We freeze the $2nd$ and $3rd$ residual block in ResNet18.
    \item \textit{FedRA}\cite{su2024fedra}: Integrates LoRA within FL architecture, where the server randomly assigns mask matrices to clients and performs weighted model aggregation.
\end{itemize}
\subsection{Performance Evaluation and Comparison Experiments}

\begin{figure}[t]
  \centering
    \begin{minipage}[b]{\linewidth}
      \subfigure[cifar10-iid-acc]{
        \includegraphics[width=0.46\linewidth]{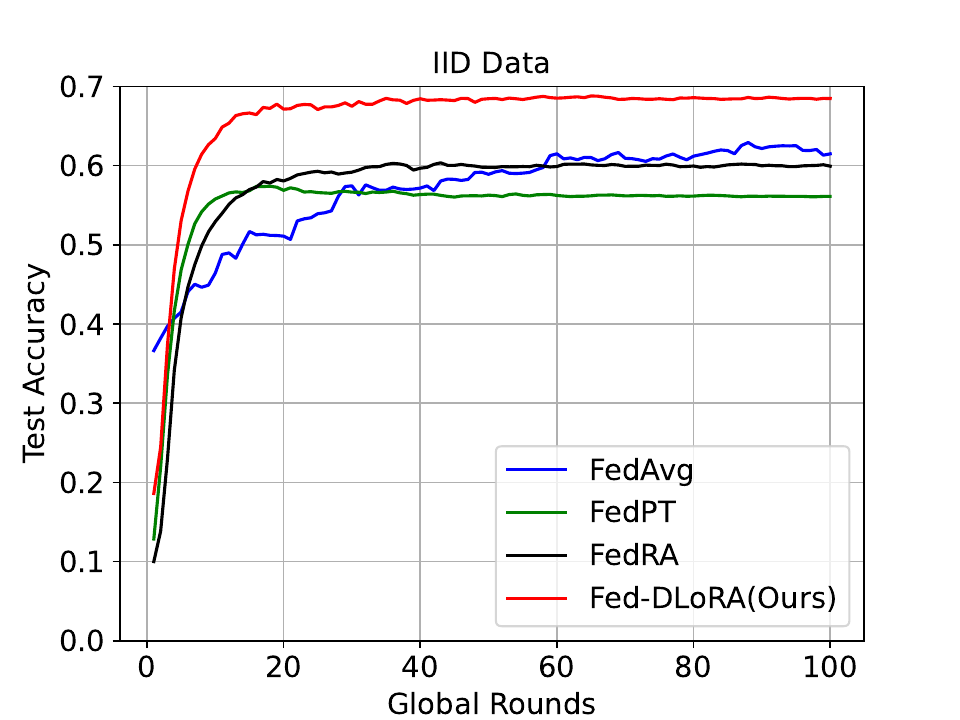}
      }
      \subfigure[cifar10-non-iid-acc]{
        \includegraphics[width=0.46\linewidth]{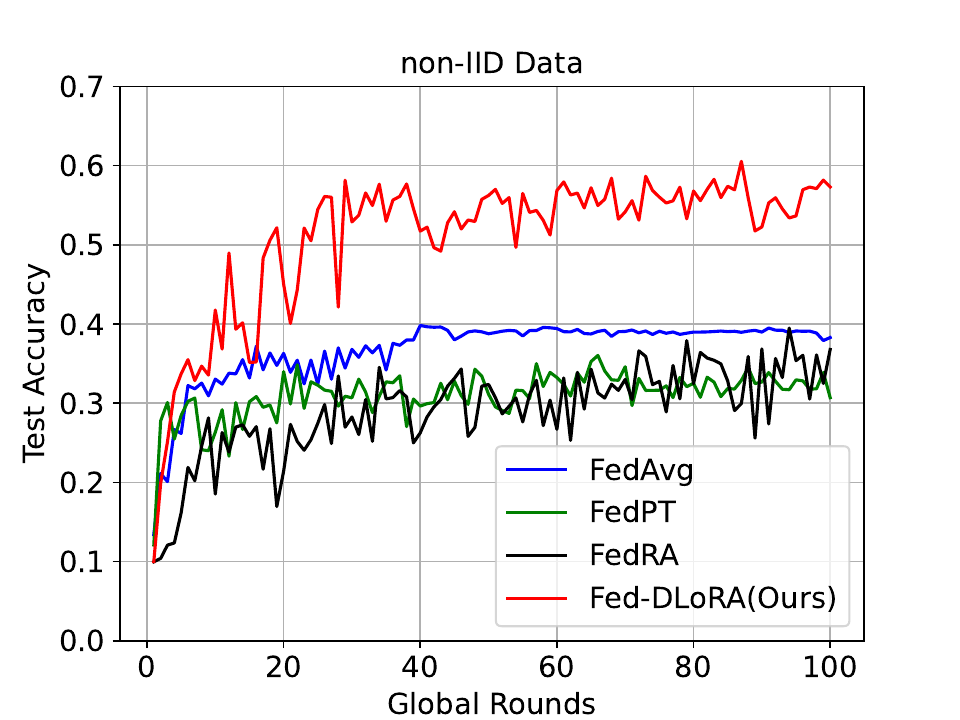}
      }
    \end{minipage}
    \caption{Test accuracy on CIFAR-10 with IID and non-IID dataset}
    \label{fig:cifar10}
\end{figure}

\begin{figure}[t]
    \centering
    \begin{minipage}[b]{\linewidth}
    \subfigure[cifar100-iid-acc]{
    \includegraphics[width=0.46\linewidth]{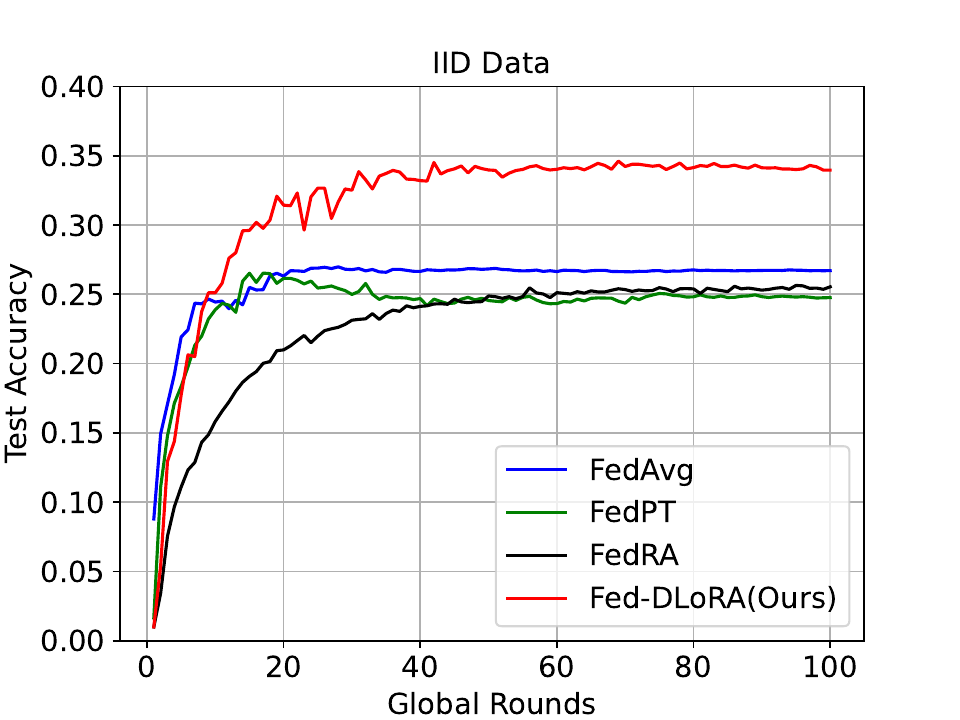}
    }
    \subfigure[cifar100-non-iid-acc]{
    \includegraphics[width=0.46\linewidth]{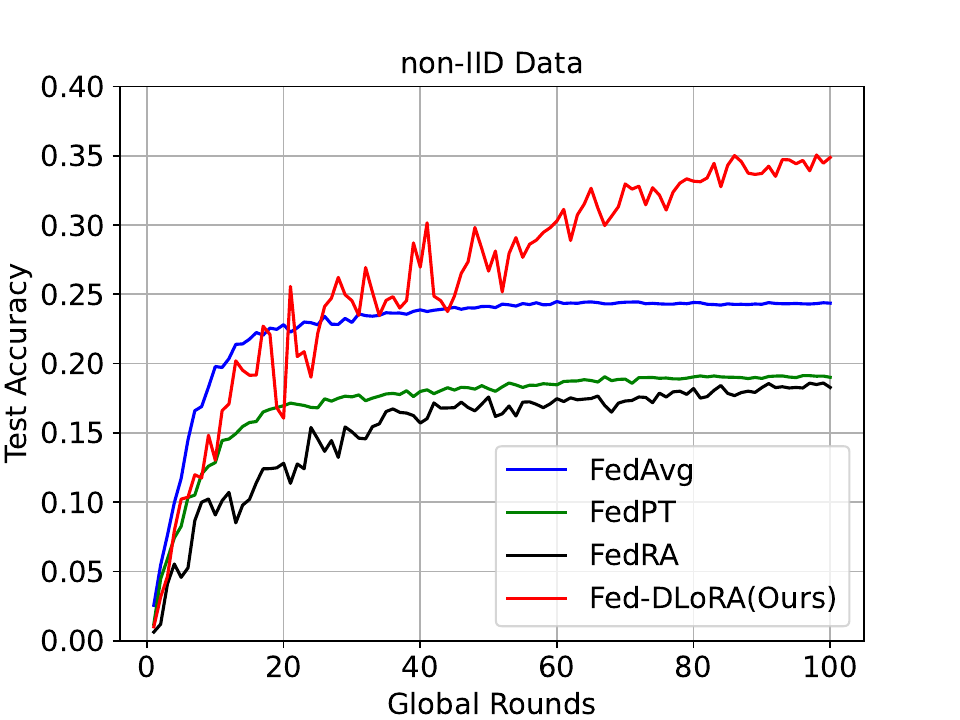}
    }
    \end{minipage}
    \caption{Test accuracy on CIFAR-100 with IID and non-IID dataset}
    \label{fig:cifar100}
\end{figure}

We systematically evaluate the performance of the Fed-DLoRA algorithm in comparison to three baseline methods. All experiments are conducted over 100 communication rounds, and the corresponding results are presented in Fig. \ref{fig:cifar10} and \ref{fig:cifar100}. 
The experimental graphs depict each model's performance, as measured by test accuracy, as a function of the communication rounds. 
Specifically, Fig. \ref{fig:cifar10}(a) and \ref{fig:cifar100}(a) display the outcomes under an IID data distribution, whereas Fig. \ref{fig:cifar10}(b) and \ref{fig:cifar100}(b) report the results under a non-IID. data distribution.

The experimental results indicate that Fed-DLoRA consistently outperforms the baseline methods by attaining higher accuracy and faster convergence across experiments.
For instance, as shown in Fig. \ref{fig:cifar10}(a), Fed-DLoRA achieves an accuracy of $66.50\%$ at the $20$-th communication round, markedly surpassing FedAvg $51.01\%$, FedPT $56.48\%$, and FedRA $57.87\%$. 
Similarly, Fig. \ref{fig:cifar100}(a) demonstrates that Fed-DLoRA exhibits superior accuracy in all rounds, except for a slight underperformance during the initial rounds. 
Moreover, in the experiments with non-IID data distributions, as illustrated in Figs. \ref{fig:cifar10}(b) and \ref{fig:cifar100}(b), Fed-DLoRA consistently shows a significant performance advantage.

Although the local model accuracy of Fed-DLoRA and FedRA is modest during the initial training phase, it subsequently shows rapid improvement before reaching convergence. 
This phenomenon is primarily attributed to the integration of the LoRA module in the training phase of both methods, which initially involves a small number of parameters in the gradient update. 
As the number of communication rounds increases, the LoRA mechanism enhances the lightness of the FL architecture, enabling more ICVs to participate in the FL process, thereby gaining richer training experience and accelerating model convergence. 
Notably, under non-IID, the participation of more ICVs brings significant performance improvement. 
This is because in the non-IID, each ICV is assigned to a random category, and the small number of ICVs participating in FL in FedAvg and FedPT may lead to certain data categories failing to participate in the FL training process, ultimately resulting in poor accuracy. 
FedRA also utilizes LoRA to achieve a lighter FL architecture, which promotes the participation of more ICVs. 
However, due to its stochastic aggregation strategy, it lacks decision-making mechanisms for factors such as optimal rank, optimal bandwidth allocation, and ICV selection, which ultimately results in suboptimal performance.
\subsection{Latency evaluation}
\begin{figure}[t]
    \centering
    \includegraphics[width=0.8\linewidth]{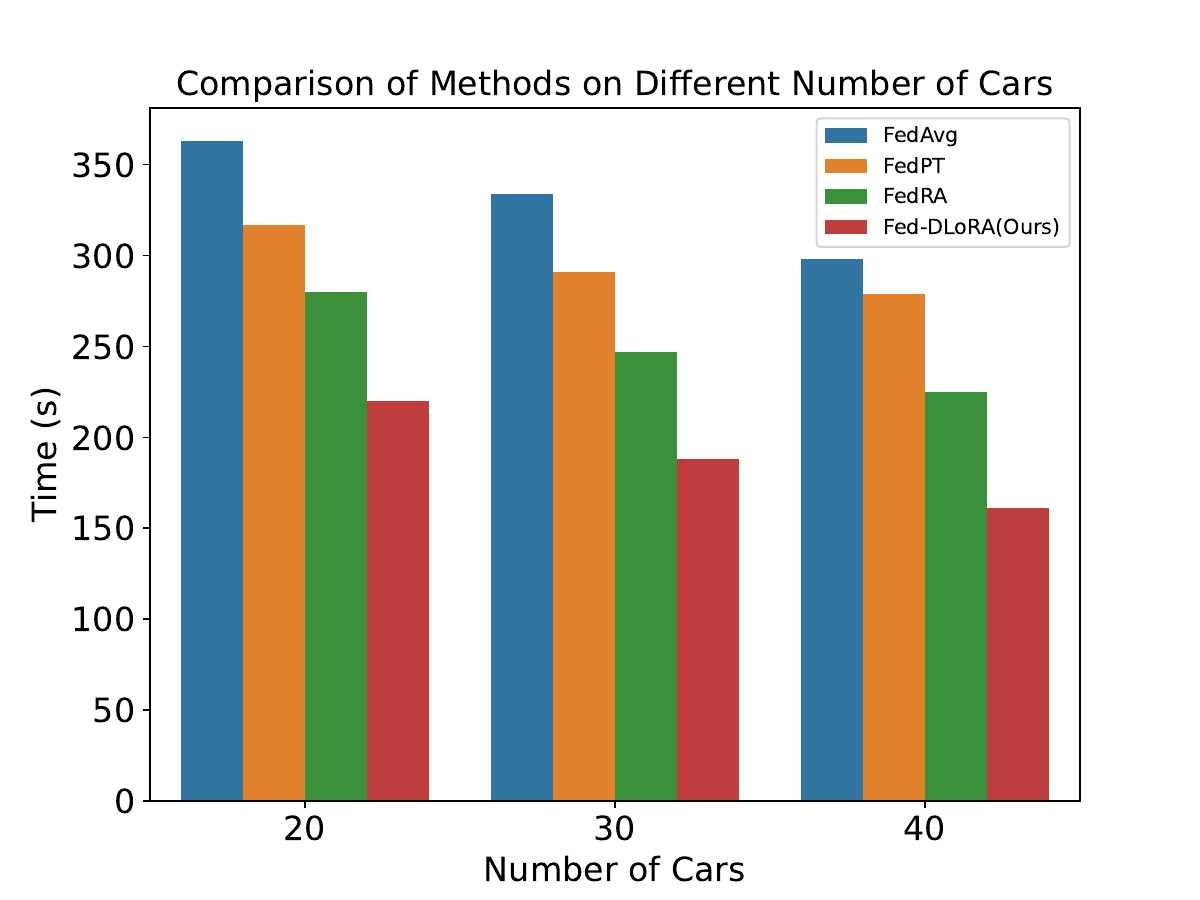}
    \caption{Comparison of time to reach target accuracy for different numbers of ICVs.}
    \label{fig:speed}
\end{figure}

Based on the comparison experiments, we further evaluate the convergence speed of the proposed model. 
Fig. \ref{fig:speed} illustrates the performance of Fed-DLoRA compared with the baseline algorithms under the IID CIFAR‑10 setting, where we measure the time required to reach $50\%$ test accuracy for different numbers of participating ICVs. 
The vertical axis represents the total elapsed time, while the horizontal axis denotes the number of ICVs in the system.

Due to the limited communication bandwidth, it is infeasible for all ICVs to participate in each FL round when running FedAvg, FedPT, and FedRA. 
Therefore, for these baseline methods, we randomly select $20\%$, $30\%$, and $50\%$ of the ICVs for training in each round, respectively, and allocate the available bandwidth evenly among the selected ICVs.

As shown in Fig. \ref{fig:speed}, Fed-DLoRA reaches the target accuracy faster than all baseline methods for every tested number of ICVs. 
Specifically, to reach the $50\%$ accuracy target, Fed-DLoRA reduces the total training time by $39.39\%$, $30.60\%$, and $21.43\%$ compared to FedAvg, FedPT, and FedRA, respectively. 
The completion time of all methods decreases slightly as the number of ICVs increases, indicating that involving more ICVs generally accelerates convergence. 
However, because the ARBVS algorithm in Fed-DLoRA adopts a greedy scheduling strategy, the growth in the number of participating ICVs does not translate linearly into a larger effective participation set, leading to only a modest reduction in the time needed to achieve the target accuracy. 
These results imply that the performance of Fed-DLoRA is mainly constrained by the total communication bandwidth and the computing capabilities of the ICVs, rather than merely by the number of available ICVs, which also highlights the robustness and stability of the proposed algorithm.

\subsection{Communication efficiency evaluation}
\begin{figure}[t]
    \centering
    \includegraphics[width=0.8\linewidth]{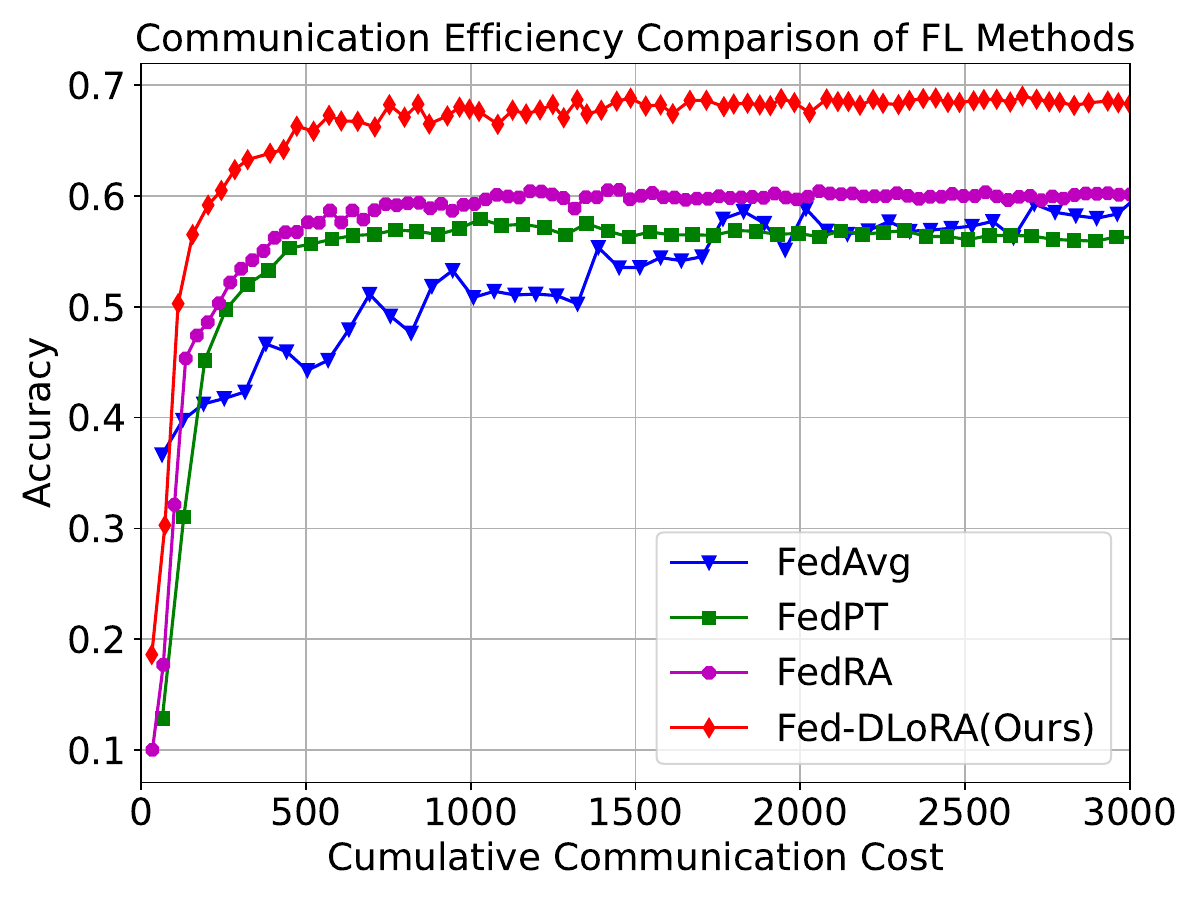}
    \caption{FL Communication Efficiency Comparison Experiment.}
    \label{fig:UplinkCost-acc}
\end{figure}
Communication efficiency is a key metric for evaluating the utilization of communication resources during model training. 
Based on the IID data distribution model of CIFAR-10, we assessed the relationship between the cumulative uplink communication cost and test accuracy, as shown as Fig. \ref{fig:UplinkCost-acc}. 
The cumulative communication cost for the uplink is defined as the total amount of data transmitted from the beginning of training to the current communication round. 
Mathematically, this is expressed as $\sum_T\sum_{\mathcal{S}} dN'$, where $T$ denotes the number of communication rounds, $\mathcal{S}$ represents the number of ICVs participating in federated learning, $d$ is the bit-width of model, and $N'$ indicates the number of model parameters.

Each data point in Fig. \ref{fig:UplinkCost-acc} represents the test accuracy of an algorithm at a specific communication round.
Due to differences in uplink transmission volumes among the algorithms, we restrict our analysis to the performance within the first $3000$ MB of data. 
To achieve $50\%$ accuracy, Fed-DLoRA saves $77.49\%$, $51.55\%$, and $33.90\%$ uplink communication cost compared to FedAvg, FedPT, and FedRA, respectively.
As illustrated in Fig. \ref{fig:UplinkCost-acc}, algorithms incorporating LoRA, Fed-DLoRA, and FedRA exhibit a denser distribution of data points, indicating a lower communication cost per round. 
Crucially, this reduction has direct implications for energy efficiency. In wireless IoV scenarios, the cumulative communication cost serves as a direct proxy for energy consumption, as transmission energy is approximately proportional to the transmitted data volume. 
Consequently, the significant reduction in communication traffic demonstrated above directly translates to energy savings. 
This confirms that our proposed method not only accelerates training but also substantially prolongs the battery life of vehicular onboard units, highlighting its practical significance for energy-constrained IoV applications.
\subsection{Performance evaluation of ICV selection algorithms}
\begin{figure}[t]
    \centering
    \includegraphics[width=0.8\linewidth]{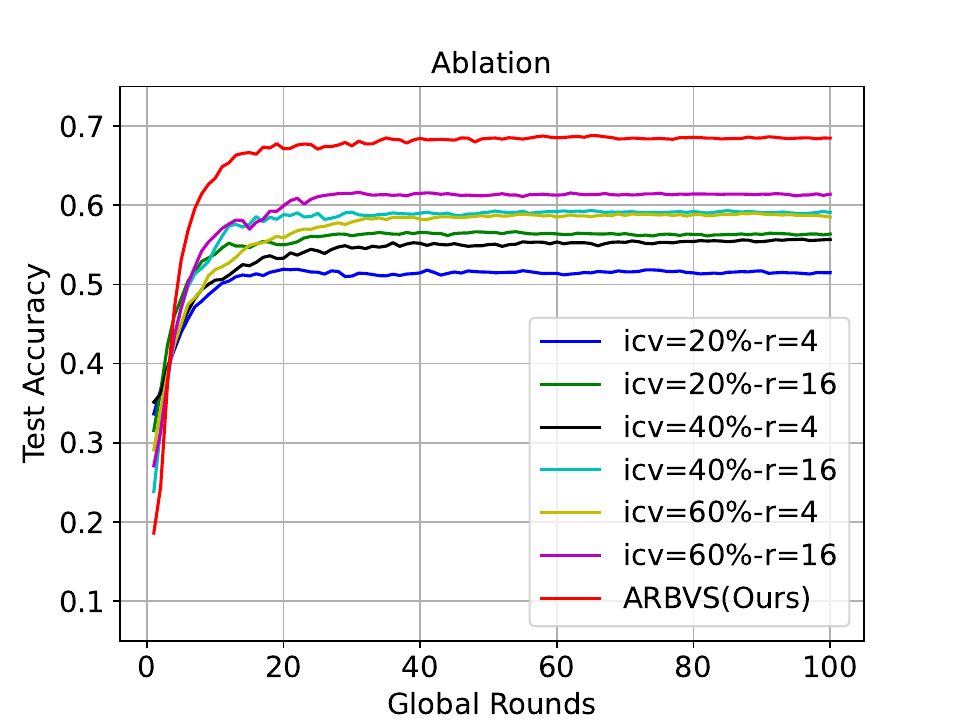}
    \caption{ICV Scheduling Comparison Experiment.}
    \label{fig:ablation_results}
\end{figure}
ARBVS is a multi-objective optimization algorithm that integrates adaptive rank selection, bandwidth allocation, and ICV selection strategy. 
To assess its performance, we compare ARBVS with the random scheduling algorithm presented in \cite{yang2019scheduling}, and the experimental results are depicted in Fig. \ref{fig:ablation_results}. 
The experiments are conducted under the IID for the CIFAR-10 dataset, with the number of ICVs fixed at $U = 20$, and employing an integrated ResNet18 model augmented with LoRA. 
For the stochastic scheduling algorithm, $20\%$, $40\%$, and $60\%$ of the ICVs are randomly selected to participate in the federated learning training, respectively. 
It is worth noting that under the strict latency constraints of IoV, the available communication bandwidth directly determines the maximum number of supportable ICVs. Therefore, comparing different ICV participation rates $(20\%, 40\%, 60\%)$ serves as an ablation study for bandwidth resources, while comparing different rank settings $(r=4,16)$ isolates the impact of the LoRA rank.

As illustrated in Fig. \ref{fig:ablation_results}, all LoRA-based training algorithms converge rapidly within approximately $20$ rounds. 
When using the same LoRA rank, an increase in the number of ICVs (implying sufficient bandwidth allocation) corresponds to a higher final test accuracy. 
Likewise, for a fixed number of ICVs (fixed bandwidth consumption), a higher LoRA rank results in improved test accuracy. 
These findings confirm the independent effects of bandwidth and rank: bandwidth improves performance by enabling larger-scale aggregation, while rank improves performance by enhancing the local model's learning capacity.
These findings corroborate the theoretical analysis presented in Section \ref{objFun}. 
Moreover, the ARBVS algorithm, which adaptively selects the adaptive LoRA rank, bandwidth allocation, and ICV configuration in each communication round, promotes the participation of more ICVs and higher LoRA rank during federated training, ultimately leading to the highest test accuracy. 
The comparative experimental results comprehensively validate the effectiveness of the ARBVS algorithm.

\section{CONCLUSION}
\label{conclusion}
In this paper, we propose the Fed-DLoRA algorithm to address the limitations of conventional federated learning regarding communication, computational efficiency, and dynamic vehicular scheduling. 
Specifically, it achieves the dual optimization of training parameter compression and communication load reduction. 
This is accomplished by introducing the LoRA and constructing lightweight sub-models, while keeping the parameters of the backbone model frozen.  
We conduct a detailed analysis of Fed-DLoRA's convergence properties and demonstrate that these properties are intrinsically linked to both the ICV scheduling strategy and the LoRA rank. 
Based on this convergence analysis, formulate a joint optimisation problem.
Facilitates the co-optimization of LoRA rank configuration, bandwidth allocation, and ICV selection, implemented via the ARBVS algorithm that employs a hybrid enumeration-greedy strategy. 
Simulation experiments conducted on public datasets within a single-base-station dynamic vehicular scenario indicate that, 
under the considered settings and model architectures, Fed-DLoRA surpasses existing algorithms such as FedAvg, FedPT, and FedRA in learning accuracy, convergence speed, and communication efficiency. We note that these findings are tied to the specific task constraints, and the performance gains are primarily attributed to the synergy between parameter-efficient updating and adaptive scheduling.

In future work, we plan to extend Fed-DLoRA to LLMs, where we will investigate joint optimization of local LoRA fine-tuning, aggregation strategies, and resource allocation to realize ubiquitous intelligence in LLM-based IoV systems.

\bibliographystyle{IEEEtran}
\bibliography{IEEEabrv,reference}

\appendix
\setcounter{equation}{0}
\renewcommand\theequation{A.\arabic{equation}}

\section*{Proof of Theorem \ref{thm:1}}
\label{appendix:proof_thm1}

This appendix provides the detailed derivation for Theorem \ref{thm:1}. Starting from the expansion in \eqref{eq:fl2C1C2C3}, the upper bound is given by:
\begin{equation}
    \mathbb{E}[\mathcal{L}(x(t+1))] \le C_1 + C_2 + C_3.
    \label{eq:appendix_start}
\end{equation}

The term $C_3$ can be expanded using the variance decomposition:
\begin{equation}
    \small
    \begin{aligned}
        C_3 =
    \frac{\eta^2 \beta}{2} \left ( \mathbb{E}\left [ \left \| \sum_{n=1}^{|\mathcal{S}|} \alpha_n (g_n(x_n(t)) - \nabla \mathcal{L}_n(x_n(t))) \right \|^2 \right] \right.\\
    \left. + \mathbb{E} \left [ \left \| \sum_{n=1}^{|\mathcal{S}|} \alpha_n \nabla \mathcal{L}_n(x_n(t)) \right \|^2 \right ] \right ).
    \end{aligned}
    \label{eq:C31}
\end{equation}
Under Assumption 3 (bounded variance), the first term in \eqref{eq:C31} is bounded by $\sigma^2 \sum \alpha_n^2$, leading to:
\begin{equation}
    \small
    \begin{aligned}
        C_3 \le \frac{\eta^2 \beta}{2} \left (\sigma^2\sum_{n=1}^{|\mathcal{S}|} \alpha_n^2 
        + \mathbb{E} \left [ \left \| \sum_{n=1}^{|\mathcal{S}|} \alpha_n \nabla \mathcal{L}_n(x_n(t)) \right \|^2 \right ] \right ).
    \end{aligned}
    \label{eq:C32}
\end{equation}

For the term $C_2$, by the linear nature of the expectation can be transformed:
\begin{equation}
    \small
    \begin{aligned}
        C_2
        &= -\eta \mathbb{E}\left[ \left< \mathbb{E}[\nabla \mathcal{L}(x(t))], \mathbb{E}\left [ \sum_{n=1}^{|\mathcal{S}|} \alpha_n g_n(x_n(t)) \right] \right> \right].
    \end{aligned}
    \label{eq:C21}
\end{equation}

Then it can be further transformed according to the unbiased gradient of Assumption2:
\begin{equation}
    \small
    \begin{aligned}
        C_2 
        &= -\eta \mathbb{E}\left[ \left< \nabla \mathcal{L}(x(t)), \sum_{n=1}^{|\mathcal{S}|} \alpha_n \nabla \mathcal{L}_n(x_n(t)) \right> \right] \\
        &= \frac{\eta}{2} \left( \mathbb{E} \left [ \left \| \nabla \mathcal{L}(x(t)) - \sum_{n=1}^{|\mathcal{S}|} \alpha_n \nabla \mathcal{L}_n(x_n(t)) \right \|^2 \right ]  \right.\\
        &\quad \left.- \mathbb{E} \left [ \left \| \nabla \mathcal{L}(x(t)) \right \|^2 \right ] - \mathbb{E} \left [ \left \| \sum_{n=1}^{|\mathcal{S}|} \alpha_n \nabla \mathcal{L}_n(x_n(t)) \right \|^2 \right ] \right).
    \end{aligned}
    \label{eq:C22}
\end{equation}
By Jensen's inequality, the first term in \eqref{eq:C22} satisfies:
\begin{equation}
    \small
    \begin{aligned}
        \mathbb{E} \left [ \left \| \nabla \mathcal{L}(x(t)) - \sum_{n=1}^{|\mathcal{S}|} \alpha_n \nabla \mathcal{L}_n(x_n(t)) \right \|^2 \right ]\\
    \le
    \sum_{n=1}^{|\mathcal{S}|} \alpha_n \mathbb{E} \left [ \left \| \nabla \mathcal{L}(x(t)) - \nabla \mathcal{L}_n(x_n(t)) \right \|^2 \right ].
    \end{aligned}
    \label{eq:C23}
\end{equation}

Substituting \eqref{eq:C32}, \eqref{eq:C22}, and \eqref{eq:C23} into \eqref{eq:appendix_start}, and grouping terms associated with $\mathbb{E} [ \| \sum \alpha_n \nabla \mathcal{L}_n(x_n(t)) \|^2 ]$, results in:
\begin{equation}
    \small
    \begin{aligned}
    \mathbb{E}[\mathcal{L}(x(t+1))] 
    &\le \mathbb{E}[\mathcal{L}(x(t))] - \frac{\eta}{2} \mathbb{E} \left[ \|\nabla \mathcal{L}(x(t))\|^2 \right] \\
    &\quad + \frac{\eta^2 \beta}{2} \sigma^2 \sum_{n=1}^{|\mathcal{S}|} \alpha_n^2 \\
    &\quad + \frac{\eta}{2} \sum_{n=1}^{|\mathcal{S}|} \alpha_n \mathbb{E} \left[ \| \nabla \mathcal{L}(x(t)) - \nabla \mathcal{L}_n(x_n(t)) \|^2 \right] \\
    &\quad + \underbrace{\left( \frac{\eta^2 \beta}{2} - \frac{\eta}{2} \right)}_{\le 0} \mathbb{E} \left[ \left\| \sum_{n=1}^{|\mathcal{S}|} \alpha_n \nabla \mathcal{L}_n(x_n(t)) \right\|^2 \right] \\
    &\le \mathbb{E}[\mathcal{L}(x(t))] - \frac{\eta}{2} \mathbb{E} \left[ \|\nabla \mathcal{L}(x(t))\|^2 \right] \\
    &\quad + \frac{\eta^2 \beta}{2} \sigma^2 \sum_{n=1}^{|\mathcal{S}|} \alpha_n^2 \\
    &\quad + \frac{\eta}{2} \sum_{n=1}^{|\mathcal{S}|} \alpha_n \mathbb{E} \left[ \| \nabla \mathcal{L}(x(t)) - \nabla \mathcal{L}_n(x_n(t)) \|^2 \right],
    \end{aligned}
    \label{eq:appendix_combined}
\end{equation}
The last term in the right-hand side of \eqref{eq:appendix_combined} is only related to the gradient of the local model, which makes it difficult to compute accurate data. Assuming that $\eta \le \frac{1}{\beta}$ , since the learning rate $\eta > 0$, then we have $\frac{\eta^2 \beta}{2} - \frac{\eta}{2} \le 0$, then the right-hand side of \eqref{eq:appendix_combined} can be further scaled up.

\end{document}